\documentclass[runningheads,orivec]{llncs}
\usepackage[T1]{fontenc}
\usepackage{graphicx}
\usepackage{booktabs}
\usepackage[misc]{ifsym}
\newcommand{\corr}{(\Letter)}

\usepackage{array}   
\usepackage{amsmath,amssymb,bm}
\usepackage{tikz}
\usetikzlibrary{calc,3d,arrows}
\usepackage{tikz-3dplot}
\usepackage[dvipsnames]{xcolor}
\usepackage{url}
\usepackage{xargs}
\usepackage{xspace}
\usepackage{algorithm}
\usepackage{algpseudocode}
\usepackage{svg}

\usepackage[colorinlistoftodos,prependcaption,textsize=small]{todonotes}

\newcommandx{\jsor}[2][1=]{\todo[ linecolor=red,backgroundcolor=red!25,bordercolor=red,#1]{\textbf{Jules:} #2}}%
\newcommandx{\ijsor}[2][1=]{\todo[inline, linecolor=red,backgroundcolor=red!25,bordercolor=red,#1]{\textbf{Jules:} #2}}%
\newcommandx{\rxd}[2][1=]{\todo[ linecolor=blue,backgroundcolor=blue!25,bordercolor=blue,#1]{\textbf{Romain:} #2}}
\newcommandx{\zak}[2][1=]{\todo[inline,linecolor=green,backgroundcolor=green!25,bordercolor=green,#1]{\textbf{Zak:} #2}}

\newcommandx{\dc}[2][1=]{\todo[inline,linecolor=brown,backgroundcolor=brown!25,bordercolor=brown,#1]{\textbf{Daniela:} #2}}
\newcommandx{\dani}[2][1=]{\todo[linecolor=brown,backgroundcolor=brown!25,bordercolor=brown,#1]{\textbf{Dani:} #2}}

\newcommandx{\ban}[2][1=]{\todo[inline, linecolor=red!25,backgroundcolor=red!10,bordercolor=red!25,#1]{\textbf{Alban:} #2}}
\newcommandx{\iban}[2][1=]{\todo[linecolor=red!25,backgroundcolor=red!10,bordercolor=red!25,#1]{\textbf{Alban:} #2}}
\newcommandx{\jgs}[2][1=]{\todo[ linecolor=orange!25,backgroundcolor=orange!10,bordercolor=orange!25,#1]{\textbf{Julieng:} #2}}
\newcommandx{\ijgs}[2][1=]{\todo[inline, linecolor=orange!25,backgroundcolor=orange!10,bordercolor=orange!25,#1]{\textbf{Julieng:} #2}}

\newif\ifshowcomments
\showcommentstrue 

\newcommand{\tcom}[2]{
    \ifshowcomments
        \ifnum0#1=1 #2 \fi
    \fi
}

\newcommand{\reals}{\mathbb{R}}
\newcommand{\inputspace}{\mathcal{F}}       
\newcommand{\latentspace}{\mathcal{Z}}      
\newcommand{\latentcomponentdim}{D}      
\newcommand{\latentindices}{\mathbf{L}}    
\newcommand{\classset}{\mathcal{C}}        
\newcommand{\prototypeindices}{\mathbf{P}} 
\newcommand{\latentrep}{\mathbf{z}}         
\newcommand{\latentrepcomp}[1]{\mathbf{z}_{#1}} 
\newcommand{\prototypej}[1]{\mathbf{p}_{#1}} 
\newcommand{\actvec}{\mathbf{a}}      
\newcommand{\actveci}[1]{\actvec_{#1}} 
\newcommand{\explanation}{\mathcal{E}}      

\newcommand{\simub}{\overline{\text{sim}}_{\explanation}}
\newcommand{\simlb}{\underline{\text{sim}}_{\explanation}}
\newcommand{\maxsimval}[1]{\overline{\actvec}_{#1}} 
\newcommand{\minsimval}[1]{\underline{\actvec}_{#1}}
\newcommand{\encoder}{f}           
\newcommand{\argmax}{\operatornamewithlimits{argmax}}
\newcommand{\softmax}{\text{softmax}}

\newcommand{\protopnet}{ProtoPNet\xspace}
\newcommand{\FXAI}{Formal XAI\xspace}

\begin{document}

\title{Beyond $L_2$: Generalizing Abductive Latent Explanations to Diverse Prototype-Based Architectures}
\titlerunning{Generalizing ALE to Diverse Prototype-Based Architectures}

\author{Jules~Soria\orcidID{0009-0009-6730-7406} \corr \and
Alban~Grastien\orcidID{0000-0001-8466-8777} \and
Romain~Xu-Darme\orcidID{0000-0002-8630-5635} \and
Julien~Girard-Satabin\orcidID{0000-0001-6374-3694} \and
Zakaria~Chihani\orcidID{0009-0004-8915-4774} \and
Daniela~Cancila\orcidID{0000-0002-3483-7947}}

\authorrunning{J. Soria et al.}



\institute{Universit\'e Paris-Saclay, CEA, List, F-91120, Palaiseau, France \\
\email{\{firstname.lastname\}@cea.fr}, \email{julien.girard2@cea.fr}}

\toctitle{Generalizing ALE to Diverse Prototype-Based Architectures}
\tocauthor{Jules~Soria,Alban~Grastien,Romain~Xu-Darme,Julien~Girard-Satabin,Zakaria~Chihani,Daniela~Cancila}

\maketitle

\begin{abstract}
Prototype-based neural networks are hailed as interpretable-by-design architectures. Recently, Abductive Latent Explanations (ALE) were introduced to provide formal, mathematically guaranteed explanations that leverage the intrinsic structure of these networks to ensure both predictive safety and human readability. ALEs rely on computing tight bounds on latent space distances to produce formal explanations. However, existing ALE formulations are rigidly confined to Euclidean latent spaces. This leaves a critical gap: modern state-of-the-art architectures increasingly rely on non-Euclidean representations---spherical metrics, Gaussian densities, and dimensional projections---rendering current formal explanation methods incompatible. In this work, we generalize the ALE framework to support non-Euclidean prototype architectures. For each geometric variant, we systematically derive how to either map the architecture to existing bounds or construct novel, architecture-specific bounding algorithms. We validate our theoretical constructions by computing subset-minimal formal explanations on fully trained image classifiers. By unifying these diverse models under a single formal framework, we enable the first rigorous, cross-architecture comparison of their interpretability. 

\end{abstract}

\newcommand{\etc}{\emph{etc.}\xspace}
\newcommand{\eg}{\emph{e.g.,}\xspace}
\newcommand{\ie}{\emph{i.e.,}\xspace}
\newcommand{\wrt}{\emph{w.r.t.}\xspace}

\section{Introduction}
\label{sec:intro}
The usefulness of Machine Learning (ML) in a variety of applications is now a widely accepted idea in computer science.
%
The importance of trustworthiness in this field, while slightly less generalized, is similarly widely understood, as evident by the numerous global efforts to reach common standards and regulations, such as the AI Act~\cite{eu_ai_act_2024}, the GDPR~\cite{eu_gdpr_24} and ISO/IEC JTC1 42~\cite{chang2022iso}.
This trustworthiness can be explored along several avenues including testing, verification and data curation.
Our work follows one of these avenues, explainability and interpretability (XAI), for which the effervescence of the state of the art leaves little doubt about its relevance~\cite{saeed2023explainable,arrieta2020explainable}.
The explicit focus placed on it by some standards (\eg AI Act, article 86 ``right to explanation''~\cite{eu_ai_act_2024,panigutti2023role}) further outlines the need for transparency in ML application, specifically when it concerns critical systems with applications in energy, transportation or the medical sector~\cite{hulsen2023explainable}.  

Surveying existing literature on XAI methods can be done along different characteristics, two of which best suit our work.
The first is their mathematical justification, where XAI can rely on strong formal logical background to further increase the trustworthiness of the given explanation~(cf Section~\ref{subsec:FXAI}).
The second one is their intervention in the development lifecycle, where XAI can be roughly categorized as post-hoc methods~\cite{ribeiro2016should,lundberg2017unified}---where analysis is done on the decisions of a pre-existing model---or by-design methods---where the model is conceived so as to offer a level of transparency in its inner workings~\cite{rudin2019stop}. 
For the latter, one fruitful field of the existing work concerns case-based reasoning, especially through prototypes~(cf Section~\ref{sec:prelim}). 

\newcommand{\pbn}{PBN\xspace}
Emerging with the seminal work of~\cite{chen2019this}, Prototype-based neural networks (PBN) justify their predictions by highlighting similar examples (``prototypes'') from the training set (``\textit{this} looks like \textit{that}''). 
Recent work~\cite{soria2026formal}---building on the formalism of \emph{abductive explanations}~\cite{marques2022delivering}---proposed algorithms to extract subset-minimal sets of prototypes that \emph{guarantee} the prediction, effectively bridging Formal XAI and \pbn in a new framework: Abductive Latent Explanations (ALE).
Where the original ALE formalism focused on Euclidean distances to evaluate similarity between prototypes,
the present work proposes to further expand the ALE formalism by showing how those principles can be generalized to numerous prototype-based architectures and activation functions~(cf Section~\ref{sec:method}).

Relying on extensive experimentation~(cf Section~\ref{sec:experiments}) to showcase the validity of our method, our contributions can be summarized as follows:
\begin{itemize}
    \item We adapt the ALE framework to other modalities of prototypes, similarity functions, and pooling layers, displaying its flexibility and usefulness.\footnote{Our work is available at \url{https://github.com/julsoria/beyond_l2}}
    \item We experimentally apply this framework leveraging the solver-free, geometric reasoning of ALEs, and compare different architectures.
\end{itemize}

\tcom{}{
\zak{
\textcircled{1}
 trustworthiness of AI is important (no need to get into "AI is important, I think it's a given) => cite some standards and regulations
 
\textcircled{2}
 XAI is part of that trust => Cite the transparency requirements from standards -- I believe it's article 50 of the AI Act, for example -- or 86
 
\textcircled{3}
 especially in safety critical but not limited to => cite the medical application of protoPNet and others, DEEL, LIME
 
\textcircled{4}
 Multiple ways of achieving that - post-hoc, by design => we pick the latter

\textcircled{5}
 Multiple levels of rigour in explainability => Criteria for XAI (not entirely sure about this one) => formal XAI

\textcircled{6}
 However, as is outlined in the related works sections, they lack blah and blah.

\textcircled{7}
 Our contributions as they answer the blah and blah. 
}

\zak{TODO: 
Cite Protopnet in medical setting @jules? 
Cite LIME and DEEL when talking about XAI}
}

\section{Preliminaries}
\label{sec:prelim}

\subsection{Prototypical-Parts Networks}


This work is performed in the context of \emph{prototypical part networks}
that classify images into one of $|\classset|$ classes.
The input of the network is an image $X$ from $\inputspace = \reals^{H_0\times W_0 \times C_0}$
where $(H_0 \times W_0)$ is the size of the image and $C_0$ is the number of channels (typically $3$).
This image is passed through a neural network \emph{backbone}~$f$ that produces a latent representation of the image
$\latentrep = \encoder(X) \in \latentspace$ where $\latentspace = \reals^{H_1 \times W_1 \times \latentcomponentdim}$.
The latent representation $\latentrep$ is then interpreted as a collection
of $(H_1 \times W_1)$ vectors $\latentrepcomp{l} \in \reals^{\latentcomponentdim}$ (or \emph{patches})
where $l \in \latentindices = \{1,\dots,H_1\} \times \{1,\dots,W_1\}$.

In a second step, the patches $\latentrepcomp{l}$ are used to \emph{activate} a set of $| \prototypeindices |$
\emph{prototypes} $\prototypej{j} \in \reals^{\latentcomponentdim}$.
These prototypes are patches from training images
that have been identified as prototypical of their respective classes.
Activation represents how similar patches of image $X$ are to these prototypes.
Activation is implemented differently in different architectures,
and the goal of this work is precisely to address different types of activation measure.
In the seminal work, \protopnet~\cite{chen2019this} starts by computing the $L_2$-distances 
$d_{l,j} = \lVert \latentrepcomp{l} - \prototypej{j} \rVert_2$
between each patch $\latentrepcomp{l}$ and each prototype $\prototypej{j}$,
and then defines the \emph{similarity} $\text{sim}(\latentrepcomp{l},\prototypej{j})$ as 
$\log \left(\frac{d_{l,j}^2 + 1}{d_{l,j}^2 + \epsilon}\right)$
where $\epsilon$ is a small constant.
Finally, the \emph{activation} $\actveci{j}$ of prototype $\prototypej{j}$ is the maximum similarity over all patches of the image: $\actveci{j}= \max_{l \in \latentindices}\ \text{sim}(\latentrepcomp{l},\prototypej{j})$.

Regardless of how the activation vector $\actvec$ is computed, 
the last step feeds the activation vector to a linear layer $W$ of dimension $|\prototypeindices| \times |\classset|$
that attributes each prototype a contributing weight for each class.
The weight $W_{j,c}$ is traditionally around $1$ if $\prototypej{j}$ is typical of class $c$ and lower than $0$ otherwise.
Finally the classification is made by selecting the class that has the highest score: $\argmax\ W^T\actvec$.


\subsection{Abductive Latent Explanation}
\label{sub:computingale}



An \emph{ALE}~\cite{soria2026formal} is a collection $\explanation$ of statements about the latent representation of the input image,
together with an implicit interpretation of this collection
that allows one to infer some information about the classification of the image.  
Such statements are made in relation with existing prototypes, 
\eg the degree of similarity between a patch and a prototype.
The implicit interpretation can indicate, for instance,
that the reported similarities are the highest ones
(or, in other words, that all other similarities are lower than the reported ones).

An ALE states some facts that are meant to be interpretable for a human (\eg a certain similarity between patches).
They are \textit{explicit} but they can imply \textit{implicit} facts too
(Section~\ref{sec:method} is largely about discovering such facts).
Consider for instance two prototypes $\prototypej{j}$ and $\prototypej{k}$ from a \protopnet model~\cite{chen2019this}.
Knowing the similarity $\text{sim}(\latentrepcomp{l},\prototypej{j})$ between $\latentrepcomp{l}$ and $\prototypej{j}$
provides some implicit information about the (dis)similarity between $\latentrepcomp{l}$ and $\prototypej{k}$.
Indeed, if the similarity is defined from a distance measure $d$ that satisfies the triangle inequality,
then the inequality 
\begin{equation}
    | d(\prototypej{j},\prototypej{k}) -
    d(\latentrepcomp{l},\prototypej{j}) |
    \ \le\ d(\latentrepcomp{l},\prototypej{k})\ \le\
    d(\prototypej{j},\prototypej{k}) +
    d(\latentrepcomp{l},\prototypej{j})
\end{equation}
holds, which makes it possible to bound the similarity $\text{sim}(\latentrepcomp{l},\prototypej{k})$. Similarly, if $d$ is the Euclidean distance, knowing $d(\latentrepcomp{l},\prototypej{j})$ and $d(\latentrepcomp{l},\prototypej{k})$ restricts the location of $\latentrepcomp{l}$ to the intersection of two hyperspheres, which can in turn be used to bound the distance between $\latentrepcomp{l}$ and a third prototype $\prototypej{m}$, as shown in~\cite{soria2026formal}.

Hence, ALEs allow one to bound similarity or activation values.
For a given explanation $\explanation$,
the bounds on the similarity between patch~$\latentrepcomp{l}$ and prototype~$\prototypej{j}$
and on the activation of prototype $\prototypej{j}$ are written, respectively,
\begin{equation}
    \text{sim}(\latentrepcomp{l},\prototypej{j}) \in [\simlb(\latentrepcomp{l},\prototypej{j}), \simub(\latentrepcomp{l},\prototypej{j})] \quad \textnormal{and} \quad \actveci{j} \in [\minsimval{\explanation,j},\maxsimval{\explanation,j}].
\end{equation}
When the exact value for a similarity or an activation is reported in the explanation,
this value explicitly trivializes the interval: $(\latentrepcomp{l},\prototypej{j}) \in \explanation \Rightarrow \simlb(\latentrepcomp{l},\prototypej{j}) = \simub(\latentrepcomp{l},\prototypej{j}) = \text{sim}(\latentrepcomp{l},\prototypej{j})$.
As shown above, this explicit information can also be used to further reduce other intervals.

Given these intervals, it is then possible to give bounds on the difference between the scores of each class. Hence, generating an ALE is an \emph{iterative} process where adding a prototype---or adding a $\langle$patch, prototype$\rangle$ pair---to the explanation $\explanation$ tightens the bounds on the activation of other prototypes, and consequently, on the output logits of the model. 
If, in this way, a single class is proved to have a greater value than all other classes, 
the decision of the network is guaranteed (\textit{explained}) and we say that the explanation $\explanation$ is \textbf{formal}. 
\emph{The goal of the present work is to extend this reasoning to prototype-based architectures that do not use the Euclidean distance as the basis for similarity.}


\section{Related Works}
\label{sec:related}

The field of XAI being vast, we deliberately limit our review of related works along two main subfields: \FXAI and prototype-based interpretable machine learning models. 
Our work is motivated by the potential for principled transparency that lies at the intersection of these two subfields. 




\subsubsection{Formal XAI}
\label{subsec:FXAI}
Formal Explainable AI (\FXAI) has emerged as a subfield that seeks to provide mathematically rigorous explanations \cite{marques2022delivering}. 
A key concept is the \emph{Abductive Explanation} (AXp) \cite{ignatiev2019abduction,boumazouza2026fame,de2025faster,ignatiev2020contrastive}, a prime implicant of the decision function~\cite{darwiche2020reasons}. Traditional \FXAI operates on input features, which is computationally NP-hard~\cite{katz2017reluplex}.
\FXAI was recently applied to computer vision---two lines of work tackle the problem of formal explanations for images: \textsc{VeriX}\cite{wu2023verix,wu2024better} produces \textit{distance-restricted} AXps to obtain arbitrarily compact explanations, ~\cite{bassan2023formal} use image segmentation methods to combine pixels into bundles (essentially superpixels) to have larger but more interpretable AXps, and ~\cite{doncenco2025dive} merged both approaches. However, those methods rely on costly solver calls and remain at the pixel-level.
\emph{Abductive Latent Explanations (ALE)} \cite{soria2026formal} lift this reasoning to the semantic latent space of prototype networks. By exploiting the geometry of the latent space (Triangular Inequality, Hypersphere Intersection Approximation), ALEs can be computed efficiently without external solvers.

\subsubsection{Advanced Prototype-Based Networks}
\label{subsec:proto}
The seminal ProtoPNet \cite{chen2019this} (along with its predecessor~\cite{li2018deep})
introduced the idea of learning prototypes in a latent space and classifying images based on squared $L_2$ distance. This stimulated a wave of variations. TesNet~\cite{wang2021interpretable} enforces non-overlapping concepts using representations on a Grassmannian manifold, and uses cosine similarity. PIP-Net~\cite{nauta2023pipnet} departs from the competition-based activation by using an approach where prototypes score independently but are normalized via softmax. ProtoPool~\cite{rymarczyk2021interpretable} and its variants introduce differentiable assignment maps and ``Focal Similarity'' to penalize background noise. More recently, probabilistic approaches~\cite{li2025interpretable,joo2025prototype,zhao2024meta,xie2025few,carmichael2024probably,liu2023learning,moradinasab2024protogmm,wang2025mixture,li2024overview,xue2024protopformer} have emerged and represent prototypes not as point-vectors but as distributions (typically von-Mises-Fischer, singular Gaussian densities or Gaussian mixtures), capturing uncertainty and variance in pattern matching. Table~\ref{tab:ppnetspp} dresses a non-exhaustive list of recent state-of-the-art prototype-based networks to highlight which components have been incrementally introduced and differ from the original ProtoPNet architecture. It reflects the changes we bring to formal explanations to support these additional components.

\begin{table}
\centering
\caption{Evolution of prototype-part network modules. Inspired by Table 1 of \cite{li2024overview}.}
\label{tab:ppnetspp}
\begin{tabular}{l@{ \ \ }lcc >{\centering\arraybackslash}p{1.8cm}}
\hline\noalign{\smallskip}
Name & Similarity & Gaussian & Pooling & Reference\\
\noalign{\smallskip}
\hline
\noalign{\smallskip}
ProtoPNet & Euclidean & $\times$ & Max & \cite{chen2019this} \\
ProtoPool & Euclidean & $\times$ & Focal & \cite{rymarczyk2021interpretable} \\
TesNet    & Cosine    & $\times$ & Max & \cite{wang2021interpretable} \\
Def. ProtoPNet & Cosine & $\times$ & Max & \cite{donnelly2022deformable} \\
PIP-Net    & Cosine\textsuperscript{$\star$} & $\times$ & Max & \cite{nauta2023pipnet} \\
ProtoGMM  & Euclidean & \checkmark & Max & \cite{moradinasab2024protogmm} \\
MGProto   & Euclidean & \checkmark & Max & \cite{wang2025mixture} \\
HyperPG   & Cosine    & \checkmark & Max & \cite{li2024overview} \\
ProtoPFormer & Euclidean & \checkmark & Max & \cite{xue2024protopformer}\\
NonParam & Cosine & $\times$ & Max & \cite{zhu2025interpretable} \\

\noalign{\smallskip}
\hline
\noalign{\smallskip}
\multicolumn{5}{l}{\scriptsize \textsuperscript{$\star$} Special case: involves dimensional projection.}
\end{tabular}
\end{table}

Finally, some prototypical-parts networks \cite{donnelly2022deformable,ayoobi2025protoargnet,moradinasab2024protogmm,wang2025mixture} rely on \emph{over}-prototypes, made up of smaller components.
Whilst the underlying principle behind the activation of such prototypes may be more complex, the problem can be reduced to a weighted sum of components' activations -- bounding the activations of components leads to bounds on the activations of prototypes.
Thus, we consider such approaches out-of-scope and focus on architectures with atomic prototypes.

\section{Theoretical Insights for new ALEs}
\label{sec:method}

\subsection{Top-$k$ Explanations}


A Top-$k$ explanation is an ALE that provides the $k$ most activated prototypes,
\textit{i.e.} that guarantees that every prototype not included in the ALE has an activation lower than the included ones.
Given an explanation $\explanation \subseteq \prototypeindices$, 
the following activation bounds can thus be computed:
\begin{equation}
    \minsimval{\explanation,j} = \left\{
      \begin{array}{ll}
          \actvec_j & \textnormal{ if } j \in \explanation \\
          \actvec^- & \textnormal{ otherwise}
      \end{array}
    \right.
    \quad
    \textnormal{and}
    \quad
    \maxsimval{\explanation,j} = \left\{
      \begin{array}{ll}
          \actvec_j & \textnormal{ if } j \in \explanation \\
          \min_{k \in \explanation}\ \actvec_k & \textnormal{ otherwise}
      \end{array}
    \right.
\end{equation}
where $\actvec^-$
is the minimal activation possible for the type of architecture of interest (for instance, $0$ for \protopnet).

\protopnet originally recommended reporting the Top-$k$ most activated prototypes for a fixed number $k$
determined before inference
(in practice, the number of prototypes positively associated with each class).
However, the number of activations that need to be included to make the explanation formal is unknown \textit{a priori}.
Thus, the iterative algorithm proposed by~\cite{soria2026formal} 
incrementally adds the prototype
whose activation is highest amongst those that haven't been included yet
until the model's prediction is guaranteed by the explanation.
Note that since this algorithm is agnostic \wrt the choice of activation function, it can be applied to \emph{all} prototype-based architectures.

\subsection{Cosine Similarity}

Models like TesNet~\cite{wang2021interpretable} and others (see Table~\ref{tab:ppnetspp}) rely on dot-product rather than Euclidean distance to measure the similarity between a latent patch and a prototype, implying that the latent space geometry is spherical. 
The similarity function used in these architectures is
$\text{sim}(\latentrepcomp{l},\prototypej{j}) = \latentrepcomp{l} \cdot \prototypej{j}$,
usually with the constraint $\forall j,~\|\prototypej{j}\| = 1$.
This similarity function has one particular drawback: it evaluates differently vectors of identical angle but different amplitudes, \ie $\text{sim}(\lambda \cdot \latentrepcomp{l},\prototypej{j}) = \lambda \cdot \text{sim}(\latentrepcomp{l},\prototypej{j})$. Since our goal is to estimate similarity bounds to guarantee the prediction, we restrict ourselves to \emph{normalized} patches, $\forall l \in \latentindices,~\|\latentrepcomp{l}\| = 1$. Thus, in the rest of this section, we assume that all vectors are on the unit sphere, and $\text{sim}(\latentrepcomp{l},\prototypej{j})\in[-1,1]$.

\subsubsection{Triangle Inequality for Cosine Similarity}
Cosine similarity\footnote{Recently, ProtoPNeXt~\cite{willard2024looks,moffett2026cosine} showed that using cosine similarity instead of Euclidean similarity is the main factor behind \pbn achieving state-of-the-art accuracy.} is not a distance metric, and the often used \emph{Cosine Distance}, defined as $d_{\cos}(x,y) = 1 - \cos(x,y)$ is not a distance metric either. It violates the triangle inequality, which is one of the spatial paradigms proposed in \cite{soria2026formal} to estimate similarity bounds between patches and prototypes.

Nevertheless, triangle inequalities for cosine similarity have been explored~\cite{schubert2021triangle} and the one that yields the tightest bounds uses the \emph{Angular Distance} $d_{\angle}$ (defined below), and can be re-expressed using only the original similarity function as
\begin{align}
    \text{sim}(x, z) &\;\geq\; \text{sim}(x, y)\cdot\text{sim}(y, z) - \sqrt{1 - \text{sim}(x, y)^2}\,\sqrt{1 - \text{sim}(y, z)^2} \\ \text{sim}(x, z) &\;\le\; \text{sim}(x, y)\cdot\text{sim}(y, z) + \sqrt{1 - \text{sim}(x, y)^2}\,\sqrt{1 - \text{sim}(y, z)^2}
\end{align}

Thus, if $(\latentrepcomp{l},\prototypej{j})\in \explanation$, then for any prototype $\prototypej{k} \notin \explanation$, $\simlb(\latentrepcomp{l},\prototypej{k})$ and $\simub(\latentrepcomp{l},\prototypej{k})$ can be expressed as functions of the known quantities $\text{sim}(\latentrepcomp{l},\prototypej{j})$ and $\text{sim}(\prototypej{k},\prototypej{j})$. 
Hence, the algorithm for generating an ALE will add pairs $\langle\latentrepcomp{l},\prototypej{j}\rangle$ incrementally, refining bounds on prototypes outside the explanation until the output of the model is guaranteed.

\subsubsection{Spherical Cap Intersection Approximation}

\begin{figure}
    \centering
    \input{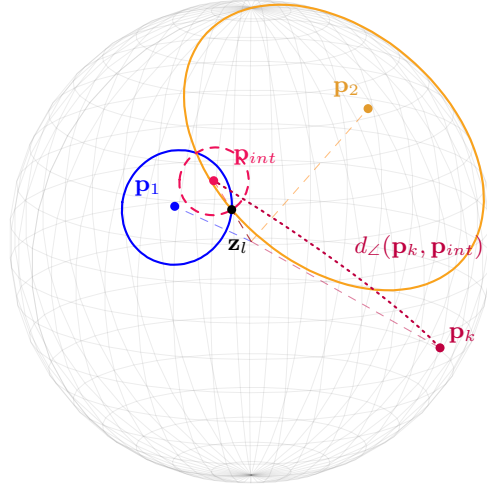}
    \caption{Spherical Cap Intersection Approximation. The \textcolor{blue}{blue} spherical cap is defined by \textcolor{blue}{$\prototypej{1}$} as its center and its radius by the distance to $\latentrepcomp{l}$. Similarly, the \textcolor{YellowOrange!90!black}{orange} spherical cap is defined by \textcolor{YellowOrange!90!black}{$\prototypej{2}$} as its center and its radius by the distance to $\latentrepcomp{l}$. Their intersection is approximated by the \textcolor{OrangeRed!90!black}{dashed red} spherical with center \textcolor{OrangeRed!90!black}{$\prototypej{int}$}, who also has its radius equal to its distance to $\latentrepcomp{l}$. With that intersection computed, we have tighter bounds on the \textit{angular} distance between an unknown prototype \textcolor{purple}{$\prototypej{k}$} and the latent patch $\latentrepcomp{}$.
    }
    \label{fig:cos_hia}
\end{figure}






\newcommand{\Sd}{\mathcal{S}^{D-1}}
\newcommand{\dang}{d_{\angle}}
\newcommand{\pint}{\mathbf{p}_{\mathrm{int}}}
\newcommand{\rint}{r_{\mathrm{int}}}
\newcommand{\pk}{\mathbf{p}_k}
\newcommand{\pone}{\mathbf{p}_1}
\newcommand{\ptwo}{\mathbf{p}_2}
\newcommand{\zz}{\mathbf{z}}

Similarly to the Hypersphere Intersection Approximation (HIA) \cite{soria2026formal} 
in the Euclidean case, we extend here this idea to the intersection of spherical caps. \emph{All proofs of theorems, lemmas, and corollaries, can be found in the Supplementary material.}

\begin{definition}[Angular Metric Space]
Let $\mathcal{M} = S^{D-1}$ be the unit hypersphere manifold. The geodesic angular distance $d_\angle(\mathbf{x}, \mathbf{y}) = \arccos(\mathbf{x} \cdot \mathbf{y})$ is a valid metric on $S^{D-1}$, satisfying non-negativity, identity of indiscernibles, symmetry, and the triangle inequality:
\begin{equation}
    d_\angle(\mathbf{x}, \mathbf{z}) \;\le\; d_\angle(\mathbf{x}, \mathbf{y}) + d_\angle(\mathbf{y}, \mathbf{z}) \quad \forall\, \mathbf{x}, \mathbf{y}, \mathbf{z} \in S^{D-1}
\end{equation}
\end{definition}

\begin{lemma}[Exact Bounding Spherical Cap via Geodesic Orthogonal Projection]
\label{lem:spherical_projection}
Let $\mathbf{p}_1, \mathbf{p}_2 \in S^{D-1}$ be two prototypes with an angular separation $0 < d_\angle(\mathbf{p}_1, \mathbf{p}_2) < \pi$. Let $\mathcal{C}_1 = \mathcal{C}(\mathbf{p}_1, \theta_1)$ and $\mathcal{C}_2 = \mathcal{C}(\mathbf{p}_2, \theta_2)$ be two spherical caps. Assume $\theta_1, \theta_2 \le \pi/2$ so that both caps are geodesically convex. Furthermore, assume the distances satisfy the spherical triangle feasibility condition $|\theta_1 - \theta_2| \leq d_\angle(\mathbf{p}_1, \mathbf{p}_2) \leq \theta_1 + \theta_2$, ensuring their intersection is non-empty.

There exists a bounding spherical cap $\mathcal{C}_{int}(\mathbf{p}_{int}, r_{int})$ such that the intersection of the original caps is entirely contained within it: $\mathcal{C}_1 \cap \mathcal{C}_2 \subseteq \mathcal{C}_{int}$.
Furthermore, $\mathcal{C}_{int}$ is the minimal-radius cap centered on the great circle through $\mathbf{p}_1$ and $\mathbf{p}_2$ (restricted Chebyshev center).
\end{lemma}


\begin{theorem}[Geometric Parameters of the Bounding Cap]
\label{thm:spherical_hia}
\begin{align}
    \alpha &= \operatorname{atan2}\Big( \cos(\theta_2) - \cos(\theta_1)\cos(d_{1,2}), \;\; \cos(\theta_1)\sin(d_{1,2}) \Big) \label{eq:alpha} \\
    r_{int} &= \arccos\left( \frac{\cos(\theta_1)}{\cos(\alpha)} \right) \label{eq:rint} \\
    \mathbf{p}_{int} &= \frac{\sin(d_{1,2} - \alpha)}{\sin(d_{1,2})} \mathbf{p}_1 + \frac{\sin(\alpha)}{\sin(d_{1,2})} \mathbf{p}_2 \label{eq:pint}
\end{align}
where $\alpha$ is the directed angular coordinate of $\mathbf{p}_{int}$ along the great circle, taking $\mathbf{p}_1$ as the origin and the direction toward $\mathbf{p}_2$ as positive. (If $\theta_1 = \pi/2$, the formula remains valid by symmetry after swapping the prototype indices).
\end{theorem}

{\renewcommand{\proofname}{Proof sketch}
\begin{proof}
Take a point $\mathbf{z}^\ast$ lying on the
boundary of both caps (at distance exactly $\theta_1$ from $\mathbf{p}_1$ and
$\theta_2$ from $\mathbf{p}_2$). Drop the geodesic perpendicular from
$\mathbf{z}^\ast$ onto the great circle $\mathcal{G}$ through
$\mathbf{p}_1,\mathbf{p}_2$; call its foot $\mathbf{p}_{int}$ and its length
$r_{int}$. This splits $\mathcal{G}$ into arcs $\alpha$ and $d_{1,2}-\alpha$ and
forms two right spherical triangles with hypotenuses $\theta_1,\theta_2$.
Napier's rule for right spherical triangles gives
$\cos\theta_1=\cos\alpha\cos r_{int}$ and
$\cos\theta_2=\cos(d_{1,2}-\alpha)\cos r_{int}$. Dividing the two eliminates
$r_{int}$ and yields $\alpha$~\eqref{eq:alpha}; back-substitution gives
$r_{int}$~\eqref{eq:rint}; and spherical linear interpolation (SLERP) along
$\mathcal{G}$ places the centre $\mathbf{p}_{int}$~\eqref{eq:pint} on the unit
sphere.
Every point of the intersection
$\mathcal{C}_1\cap\mathcal{C}_2$ lies within $\mathcal{C}(\mathbf{p}_{int},r_{int})$,
and no cap centred on $\mathcal{G}$ with a smaller radius can contain the
intersection. Hence $\mathcal{C}(\mathbf{p}_{int},r_{int})$ is the smallest such
enclosing cap (its restricted Chebyshev centre).
Intersecting with a further
prototype repeats step~(1) with the \emph{previous} radius as the hypotenuse and
the new radius as a leg, so $\cos r_{int}^{(i+1)}\ge\cos r_{int}^{(i)}$, i.e.
$r_{int}^{(i+1)}\le r_{int}^{(i)}$: the bounding region never grows.
Since the unknown feature lies
in the cap, $d_\angle(\mathbf{p}_{int},\mathbf{z})\le r_{int}$. The triangle
inequality then confines its distance to any target $\mathbf{p}_k$ to
$\big[\,|\delta_k-r_{int}|,\ \delta_k+r_{int}\,\big]$ with
$\delta_k=d_\angle(\mathbf{p}_k,\mathbf{p}_{int})$. Because $\cos$ is strictly
decreasing on $[0,\pi]$, applying it flips these into the explicit lower and
upper bounds on the cosine similarity.
The construction yields closed-form, provably sound
similarity bounds that tighten monotonically as evidence is added, computed
entirely on the sphere.
\end{proof}}

\begin{corollary}[Monotonic Decrease of the Bounding Radius]
\label{cor:monotonic_shrinkage}
Let $\mathcal{C}_{int}^{(i)}$ be the current restricted Chebyshev bounding cap after adding $i$ prototypes to the explanation. Intersecting it with a newly added prototype cap $\mathcal{C}_{new}(\mathbf{p}_{new}, \theta_{new})$ via the orthogonal projection method yields a subsequent bounding cap $\mathcal{C}_{int}^{(i+1)}$. The sequence of bounding radii is monotonically decreasing: $r_{int}^{(i+1)} \leq r_{int}^{(i)}$.
\end{corollary}


\begin{corollary}[Cosine Similarity Bounds via Monotone Inversion]
\label{cor:cosine_bounds}
For any arbitrary target prototype $\prototypej{k}$ outside the explanation,
the cosine similarity ${\normalfont \text{sim}}(\latentrepcomp{l},\prototypej{k}) = \latentrepcomp{l} \cdot \prototypej{k}$
is bounded by:
\begin{align}
    \normalfont \simlb(\latentrepcomp{l},\prototypej{k})
        &= \cos\!\Big(\min\!\big(\pi,\;\dang(\prototypej{k}, \pint) + r_{int}\big)\Big),
    \,\\
    \normalfont \simub(\latentrepcomp{l},\prototypej{k})
        &= \cos\!\Big(\dang(\prototypej{k}, \pint) - r_{int}\Big).
\end{align}
\end{corollary}


Similar to the Euclidean case, Lemma~\ref{lem:spherical_projection} and Corollary~\ref{cor:monotonic_shrinkage} indicate that the intersection of two spherical caps is included in a spherical cap of smaller radius (as illustrated on Figure~\ref{fig:cos_hia}). Hence, the addition of a prototype cap $\mathcal{C}_{new}(\mathbf{p}_{new}, \theta_{new})$ to the ALE creates a new intersection $(\prototypej{int}, r_{int})$, which in turn refines the similarity bounds between a latent patch $\latentrepcomp{l}$ and any $\prototypej{k}\notin \explanation$ (Corollary~\ref{cor:cosine_bounds}).
This gives us bounds on the activation of the prototypes and can guarantee (or not) the given prediction. 

\subsection{Dimensional Projection}
PIP-Net employs a \emph{softmax} function over the prototypes for a given patch, essentially mapping activations to a probability simplex $\Delta^{m-1}$. Each dimension of that vector represents its similarity to the indexed prototype (a prototype is a dimension in this case), \ie $\prototypeindices = \{1,\dots,\latentcomponentdim\}$.
More precisely, for each location $l\in \latentindices$ and for each prototype $j\in\prototypeindices$, the similarity value $\text{sim}(\latentrepcomp{l},\prototypej{j})$ is defined as 
\begin{equation}
    \hat{\latentrepcomp{}}_{l} = \softmax(\latentrepcomp{l}),\quad \text{sim}(\latentrepcomp{l},\prototypej{j}) = \hat{\latentrepcomp{}}_{l,j}
\end{equation}

\subsubsection{Simplex Explanation}
Due to the softmax, the spatial explanation is no longer geometric; it is a \textbf{conservation of probability mass}, with
$\text{sim}(\latentrepcomp{l},\prototypej{j})\in [0,1]$, and  $\sum_{j \in \prototypeindices}\text{sim}(\latentrepcomp{l},\prototypej{j}) = 1$. 
 It follows that $\actveci{j}= \max_{l}\ \text{sim}(\latentrepcomp{l},\prototypej{j}) \in [0,1]$ and $\actvec^-=0$ (used for Top-$k$ explanations).
In this paradigm, adding certain $\langle$patch,prototype$\rangle$ pairs to the ALE significantly constrains the remaining ones:
\begin{equation}
\simub(\latentrepcomp{l},\prototypej{j}) =
1-\sum_{\langle\latentrepcomp{l},\prototypej{k}\rangle \in \explanation} \text{sim}(\latentrepcomp{l},\prototypej{k})
\end{equation}

In other words, by adding $\langle\latentrepcomp{l},\mathbf{p}_{j}\rangle$ to the explanation $\explanation$, we ``consume'' a portion of the softmax probability for this specific patch. The absolute upper bound for any prototype outside the explanation across the entire image becomes strictly limited by the maximum remaining residual mass:
\begin{equation}
    \maxsimval{k}
    \le \max_{l} \left( 1 - \sum_{\langle\latentrepcomp{l},\prototypej{j}\rangle \in \explanation} \text{sim}(\latentrepcomp{l},\prototypej{j}) \right)
\end{equation}


\subsubsection{Sparse-Weight Explanation}

Traditional ProtoPNets have dense classification heads, \textit{i.e.} every prototype slightly influences every class. On the contrary, PIP-Net's decision head is incentivized to be a very sparse matrix, with non-negative weights ($w_{j,c} \ge 0$).
In this case, we start by initializing the explanation $\explanation$ with the set of all prototypes that have a non-zero weight for the predicted class $\hat{c}$:

\begin{equation}
\explanation_0 = \{ j \mid w_{j,\hat{c}} > 0 \}
\end{equation}
Because all other prototypes have a weight of $0$ for $\hat{c}$, the prototypes inside $\explanation_0$ give us the exact score for the predicted class. And for any explanation $\explanation \supseteq \explanation_0$:
\begin{equation}
 o_{\hat{c}} = \sum_{j\in \prototypeindices} w_{j,\hat{c}} \cdot \mathbf{a}_j = \sum_{j \in \explanation} w_{j,\hat{c}} \cdot \mathbf{a}_j = \sum_{j \in \explanation_0} w_{j,\hat{c}} \cdot \mathbf{a}_j
\end{equation}

Since maximum activations are globally bounded by $1.0$ (due to the softmax), the output value $o_{c'}$ for any competing class $c'\neq \hat{c}$ is bounded by:
\begin{equation}
o_{c'} \leq \overline{o}_{c'} = \sum_{j \in \explanation} w_{j,c'} \cdot \mathbf{a}_j + \sum_{j' \notin \explanation} w_{j',c'} \cdot 1.0
\end{equation}
If $\overline{o}_{c'} < o_{\hat{c}}$ for all competing classes $c'\neq \hat{c}$, then the prediction is formally guaranteed. Otherwise, for $c^*=\argmax_{c'\neq \hat{c}}\overline{o}_{c'}$, we add to $\explanation$ the activation score that contributes the most to reducing $\overline{o}_{c^*}$, \textit{i.e.}
\begin{equation}
    \explanation \gets \explanation \cup \argmax_{j \notin \explanation} (1-\mathbf{a}_j)w_{j,c^*}
\end{equation}

\subsection{Isotropic Gaussian Similarity}
Gaussian networks replace the distance $d(\latentrep, \prototypej{j})$ with a negative log-likelihood $-\log P(\latentrep | \prototypej{j})$.
For prototype $\prototypej{j} \sim \mathcal{N}(\mu_j, \Sigma_j)$, if we assume an isotropic covariance $\Sigma_j = \sigma_j^2 \text{I}$ for each prototype, we can recover the ``true'' Euclidean distance between the $l$-th latent patch $\latentrepcomp{l}$ and the $j$-th prototype $\prototypej{j}$ and even express it on the metric space defined by the $k$-th prototype $\prototypej{k}$. Concretely, $\forall j\in\prototypeindices$:

\begin{equation}
    d_{\text{Euclid}}(x,y) \,=\, \sigma_j\, d_{\Sigma_j}(x,y)
\end{equation}

Because the transformation from a metric distance to a similarity score is strictly monotonically decreasing, our bounding logic is entirely agnostic to the specific choice of similarity function. Whether the model uses a Gaussian probability density $S(d) = -\frac{1}{2}d^2 - \log Z$, or a heavy-tailed distribution like the legacy ProtoPNet $S(d) = \log\big(\frac{d^2+1}{d^2+\epsilon}\big)$, we can construct formal bounds by mapping the problem into a universal Euclidean space. 

This provides a powerful geometric bridge: it allows us to map similarities into Euclidean radii, natively apply the Hypersphere Intersection Approximation (HIA)~\cite{soria2026formal}, and then remap the tight geometric bounds back into the model's specific activation space (see Figure~\ref{fig:prob_hia}). The end-to-end bounding pipeline proceeds as follows:

\paragraph{Inverse Mapping to Metric Space.} Given sim($\latentrepcomp{l},\prototypej{j}$), 
we apply the inverse similarity function $S^{-1}(\cdot)$ to extract the target-metric distance $d_{\Sigma_j}$, and scale it to recover the exact Euclidean distance\footnote{Note: For standard unscaled models like ProtoPNet, $\forall j.\,\sigma_j = 1$.}:
\begin{equation}
    d_{\text{Euclid}}(\latentrepcomp{l}, \mu_j) = \sigma_j \, S^{-1}\left(\text{sim}(\latentrepcomp{l},\prototypej{j})\right)
\end{equation}

\paragraph{Hypersphere Intersection Approximation.} Because the Gaussian latent space is Euclidean, the observed distance $d_{\text{Euclid}}(\latentrepcomp{l}, \mu_j)$ defines the exact radius of a standard Euclidean hypersphere centered at the prototype's anchor $\mu_j$. The true latent patch $\latentrepcomp{l}$ must lie exactly on the intersection of the surfaces of these hyperspheres. We iteratively apply the HIA algorithm to this sequence of observed hyperspheres. This geometric operation computes a new bounding Euclidean hypersphere, defined by a center $\mathbf{c}_{int}$ and a minimal enclosing radius $r_{int}$ that strictly encapsulates the exact intersection (see Figure~\ref{fig:prob_hia}).

\paragraph{Target-Metric Projection.} To bound the distance to a prototype $\prototypej{k}\notin \explanation$, we measure the Euclidean distance between our newly computed bounding center $\mathbf{c}_{int}$ and that prototype, and scale it back into its metric space. The distance $d_{\Sigma_k}(\latentrepcomp{l}, \prototypej{k})$ is bounded by:
\begin{align}
    \quad & \overline{d}_{\Sigma_k}(\latentrepcomp{l}, \prototypej{k}) \,=\, \frac{1}{\sigma_k} \Big( \|\mathbf{c}_{int} - \mu_k\|_2 + r_{int} \Big) \\
    \quad & \underline{d}_{\Sigma_k}(\latentrepcomp{l}, \prototypej{k}) \,=\, \frac{1}{\sigma_k} \max\Big(0, \; \|\mathbf{c}_{int} - \mu_k\|_2 - r_{int} \Big)
\end{align}

\paragraph{Forward Mapping to Activations.} Finally, because any valid similarity function $S(d)$ must be monotonically decreasing with respect to distance, we strictly invert the geometric bounds. The maximum possible distance yields the lower bound on similarity, and the minimum possible distance yields the upper bound:
\begin{equation}
    \simlb(\latentrepcomp{l},\prototypej{k}) = S\big(\overline{d}_{\Sigma_k}\big) 
    \quad
    \textnormal{and}
    \quad
    \simub(\latentrepcomp{l},\prototypej{k}) = S\big(\underline{d}_{\Sigma_k}\big)
\end{equation}
This pipeline completely decouples the geometric intersection logic from the model's forward pass, allowing subset-minimal explanations to be computed rapidly for a wide variety of prototype architectures.

\begin{figure}
    \centering
\begin{tikzpicture}[scale=0.27, line join=round, line cap=round]

  \fill[blue, opacity=0.32] (-5.458,5.292) circle (3.970);
  \fill[YellowOrange, opacity=0.32] (4.500,8.000) circle (2.007);
  \fill[purple, opacity=0.32] (-2.058,-6.058) circle (1.200);
  \fill[blue, opacity=0.24] (-5.458,5.292) circle (7.939);
  \fill[YellowOrange, opacity=0.24] (4.500,8.000) circle (4.015);
  \fill[purple, opacity=0.24] (-2.058,-6.058) circle (2.400);
  \fill[blue, opacity=0.16] (-5.458,5.292) circle (11.909);
  \fill[YellowOrange, opacity=0.16] (4.500,8.000) circle (6.022);
  \fill[purple, opacity=0.16] (-2.058,-6.058) circle (3.600);

  \draw[blue, very thick] (-5.458,5.292) circle (7.964);
  \draw[YellowOrange, very thick] (4.500,8.000) circle (6.042);
  \draw[OrangeRed, very thick, dashed] (0.779,6.988) circle (4.651);

  \draw[blue, dashed] (-5.458,5.292) -- (-11.090,10.923);
  \draw[YellowOrange, dashed] (4.500,8.000) -- (8.772,12.272);
  \draw[OrangeRed, ultra thick] (-2.058,-6.058) -- (-0.209,2.443);
  \draw[darkgray, thick, dotted] (-2.058,-6.058) -- (1.768,11.533);

  \node[circle, fill=blue, inner sep=2pt] at (-5.458,5.292) {};
  \node[circle, fill=YellowOrange, inner sep=2pt] at (4.500,8.000) {};
  \node[circle, fill=purple, inner sep=2pt] at (-2.058,-6.058) {};
  \node[circle, fill=black, inner sep=2pt] at (2.000,2.500) {};
  \draw[OrangeRed, thick] (0.479,6.688) -- (1.079,7.288);
  \draw[OrangeRed, thick] (0.479,7.288) -- (1.079,6.688);

  \node[blue, below] at (-5.500,4.800) {$\boldsymbol{\mu_1}$};
  \node[YellowOrange!50!black, below] at (4.400,7.700) {$\boldsymbol{\mu_2}$};
  \node[purple, below] at (-2.058,-6.858) {$\boldsymbol{\mu_k}$};
  \node[OrangeRed, above] at (2.500, 0.500) {$\mathbf{z}_l$};
  \node[OrangeRed, below] at (-0.179,10.388) {$\mathbf{C_{int}}$};
\end{tikzpicture}
    \caption{Gaussian Hypersphere Intersection Approximation. Each prototype $\prototypej{j}$ is defined by its anchor $\mu_j$ and its variance $\sigma_j$ (represented as a layer of the same color). In this figure, $\latentrepcomp{l}$ is at distance $2$ to \textcolor{blue}{$\boldsymbol{\mu_1}$} and distance $3$ to \textcolor{orange!75!black}{$\boldsymbol{\mu_2}$}. The red dotted circle encompasses the hypersphere that approximates the intersection between the two hyperspheres in the Euclidean metric space.}
    \label{fig:prob_hia}
\end{figure}
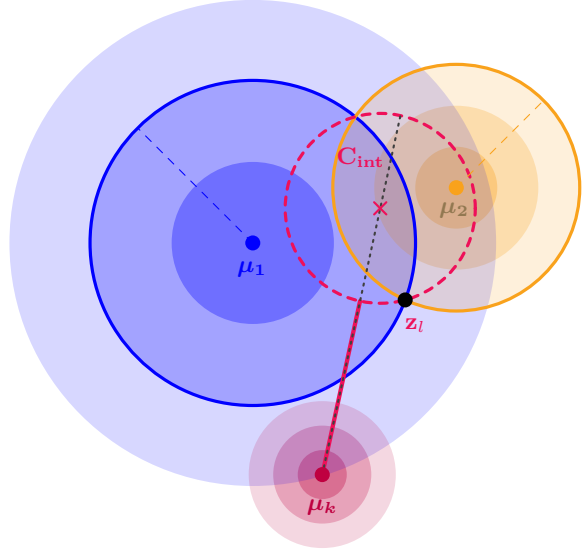

\subsection{Focal Similarity}
ProtoPool \cite{rymarczyk2021interpretable} employs a ``Focal Similarity'' to suppress background activations and encourage prototypes to be localized. Technically, this Focal Similarity can be used in tandem with any predefined (or novel) similarity function, as it only changes the pooling before the final prototypes' activations. It effectively replaces the \texttt{maxpooling} ($\actveci{j} = \max_{l \in \latentindices}\, \text{sim}(\latentrepcomp{l}, \prototypej{j})$) used in other architectures with a relative \texttt{focal pooling} ($\actveci{j} = \max_{l \in \latentindices}\, \text{sim}(\latentrepcomp{l}, \prototypej{j}) - \underset{l \in \latentindices}{\mathbb{E}} [\text{sim}(\latentrepcomp{l}, \prototypej{j})]$) taking into account (and minimizing) the average activation of the prototype latent representation.


To bound this within the ALE framework, we leverage the fact that spatial ALEs already provide formal similarity bounds $[\simlb(\latentrepcomp{l}, \prototypej{j}), \simub(\latentrepcomp{l}, \prototypej{j})]$ for \emph{every} latent patch $l \in \latentindices$.
Consequently, we can derive strict bounds for the aggregate statistics required by focal pooling. For a given prototype $\prototypej{j}$, the maximum and expected patch similarities are bounded by:
\begin{align}
    \max_{l \in \latentindices} \text{sim}(\latentrepcomp{l}, \prototypej{j}) &\in \left[ \max_{l \in \latentindices} \simlb(\latentrepcomp{l}, \prototypej{j}), \max_{l \in \latentindices} \simub(\latentrepcomp{l}, \prototypej{j}) \right] \\
    \underset{l \in \latentindices}{\mathbb{E}} \left[\text{sim}(\latentrepcomp{l}, \prototypej{j})\right] &\in \left[ \underset{l \in \latentindices}{\mathbb{E}} [\simlb(\latentrepcomp{l}, \prototypej{j})], \underset{l \in \latentindices}{\mathbb{E}} \left[\simub(\latentrepcomp{l}, \prototypej{j})\right] \right]
\end{align}

Because the focal activation $\actveci{j}$ is defined as the difference between these two terms, its extreme values occur at opposing bounds. We obtain the final, formally guaranteed activation bounds by cross-subtracting the extrema:
\begin{align}
    \minsimval{\explanation, j} &= \max_{l \in \latentindices} \simlb(\latentrepcomp{l}, \prototypej{j}) - \underset{l \in \latentindices}{\mathbb{E}} \left[\simub(\latentrepcomp{l}, \prototypej{j})\right] \\
    \maxsimval{\explanation, j} &= \max_{l \in \latentindices} \simub(\latentrepcomp{l}, \prototypej{j}) - \underset{l \in \latentindices}{\mathbb{E}} \left[\simlb(\latentrepcomp{l}, \prototypej{j})\right]
\end{align}

This formulation seamlessly extends spatial ALEs to support ProtoPool. By aggregating the existing patch-wise similarity bounds, it handles focal similarity penalties without requiring any changes to the underlying geometric solver.

\section{Experiments}
\label{sec:experiments}

To empirically validate our framework, we implement the aforementioned spatial reasoning techniques and extract subset-minimal explanations that yield rigid mathematical guarantees on model predictions. We trained four distinct architectures\footnote{We also trained ProtoPool on CUB200 but we did not succeed at training this architecture on the other two datasets.}---a ProtoPNet, a PIP-Net, a TesNet, and an Isotropic Gaussian ProtoPNet---on three fine-grained image classification datasets: Oxford Flowers 102~\cite{Nilsback08}, Oxford IIIT Pet~\cite{parkhi12a}, and CUB200~\cite{wah2011caltech}. All experiments were executed using the CaBRNet\footnote{\url{https://github.com/aiser-team/cabrnet}} library~\cite{xu2024cabrnet}.

For each formal explanation method, we report three key metrics: absolute size, relative size, and computing time. We prioritize these metrics because the time required to compute an explanation is a critical operational bottleneck for practitioners selecting XAI methods, and the size of the explanation is intrinsically tied to human cognitive limits and interpretability~\cite{miller1956magical,fox2024cognitive}, and had been argued to be a factor of acceptance for humans~\cite{miller2019explanation,nauta2023anecdotal}. 

To provide a fair comparison across fundamentally different ALE paradigms, we introduce the Relative Size metric. This normalizes the absolute size of an explanation by the maximum possible theoretical size for that specific paradigm:
\begin{itemize}
    \item \textbf{Global Paradigms} (Top-$k$, Sparse-Weight): These methods rely on raw prototype activations. Their maximum size is strictly bounded by the total number of prototypes in the network, $|\prototypeindices|$ 
    \item \textbf{Spatial Paradigms} (Simplex, TI, HIA): These methods evaluate prototype similarities across the entire latent grid. Their minimum size is $|\latentindices|$, and their maximum size scales with the feature map dimensions, $|\latentindices| \times |\prototypeindices|$.
\end{itemize}

\paragraph{Training Details.}
The total number of training epochs ranged from 30 epochs to 100 depending on the architecture and the dataset. A warm-up phase is used at first to train the prototype layer but not the ``backbone'' i.e. our feature extractor. A projection phase interrupts the main training periodically to push the prototypes to latent patches of the training dataset images.

\paragraph{Results.} 
The performance of the diverse ALE methods across both datasets is reported in Table~\ref{tab:comparison_oxford_flowers} and Table~\ref{tab:comparison_oxford_pets} (best result across ALEs for a specific architecture in \textbf{bold}, overall \underline{underlined}). The empirical results confirm that our generalized spatial bounding techniques successfully compress the relative hypothesis space, though often at the cost of computational overhead. We rigorously dissect the architectural inductive biases driving these trade-offs in Section~\ref{sec:discussion}.

\begin{table}[ht!]
\centering
\caption{Comparison of explanation size and computing time on Oxford Flowers 102.}
\label{tab:comparison_oxford_flowers}
\begin{tabular}{l c l ccc}
\hline\noalign{\smallskip}
Architecture & Acc. (\%) & ALE & Abs. Size & Rel. Size (\%) & Time (s) \\
\noalign{\smallskip}\hline\noalign{\smallskip}
ProtoPNet & $82.8$
 & Top-$k$ & $\mathbf{199} \pm 254$ & $19.5 \pm 24.9$ & $\mathbf{0.327} \pm 0.406$ \\
 & & TI & $5273 \pm 6075$ &  $32.3 \pm 37.2$ & $60.434 \pm 61.289$ \\
 & & HIA & $675 \pm 361$ & $\mathbf{4.1} \pm 2.2$ & $11.781 \pm 13.676$ \\

\noalign{\smallskip}\hline\noalign{\smallskip}
PIP-Net & $80.5$
 & Top-$k$ & $135 \pm 29$ & $17.6 \pm 3.8$ & $1.229 \pm 1.413$ \\
 & & Sparse & $\underline{\mathbf{103}} \pm 55$ & $13.5 \pm 7.2$ & $\underline{\mathbf{0.075}} \pm 0.089$ \\
 & & Simplex & $250 \pm 103$ & $\underline{\mathbf{0.2}} \pm 0.1$ & $0.286 \pm 0.069$ \\
 
\noalign{\smallskip}\hline\noalign{\smallskip}
TesNet & $81.9$
 & Top-$k$ & $295 \pm 166$ & $28.9 \pm 16.3$ & $\mathbf{0.341} \pm 0.224$ \\
 & & Cosine TI & $10438 \pm 2862$ & $64.0 \pm 17.5$ & $202.580 \pm 64.435$ \\
 & & Spherical HIA & $\mathbf{187} \pm 51$ & $\mathbf{1.1} \pm 0.3$ & $4.491 \pm 2.827$ \\
\noalign{\smallskip}\hline\noalign{\smallskip}
Gaussian & $72.8$
 & Top-$k$ & $\mathbf{257} \pm 237$ & $25.2 \pm 23.2$ & $\mathbf{0.651} \pm 0.563$ \\
 & & Scaled TI & $1022 \pm 2035$ & $6.3 \pm 12.5$ & $20.490 \pm 44.649$ \\
 & & Scaled HIA & $931 \pm 1770$ & $\mathbf{5.7} \pm 10.8$ & $157.373 \pm 312.566$ \\
\end{tabular}
\end{table}

\begin{table}[ht!]
\centering
\caption{Comparison of explanation size and computing time on Oxford IIIT Pets.}
\label{tab:comparison_oxford_pets}
\begin{tabular}{l c l ccc}
\hline\noalign{\smallskip}
Architecture & Acc. (\%) & ALE & Abs. Size & Rel. Size (\%) & Time (s) \\
\noalign{\smallskip}\hline\noalign{\smallskip}
ProtoPNet & $89.3$
 &   Top-$k$ & $\mathbf{99} \pm 72$ & $26.6 \pm 19.4$ & $\mathbf{0.101} \pm 0.06$ \\
 & & TI & $7869 \pm 6537$ & $43.4 \pm 36.1$ & $105.805 \pm 83.69$ \\
 & & HIA & $733 \pm 204$ & $\mathbf{4.0} \pm 1.1$ & $3.098 \pm 1.49$ \\

\noalign{\smallskip}\hline\noalign{\smallskip}
PIP-Net & $88.7$
 &   Top-$k$ & $131 \pm 13$ & $17.0 \pm 1.7$ & $1.586 \pm 0.231$ \\
 & & Sparse & $\mathbf{30} \pm 16$ & $4.0 \pm 2.0$ & $\mathbf{0.076} \pm 0.038$ \\
 & & Simplex & $190 \pm 30$ & $\underline{\mathbf{0.1}} \pm 0.0$ & $0.252 \pm 0.040$ \\
 
\noalign{\smallskip}\hline\noalign{\smallskip}
TesNet & $93.2$
 &   Top-$k$ & $\underline{\mathbf{16}} \pm 18$ & $4.4 \pm 4.8$ & $\underline{\mathbf{0.027}} \pm 0.011$ \\
 & & Cosine TI & $1750 \pm 2149$ & $9.7 \pm 11.9$ & $26.565 \pm 41.140$ \\
 & & Spherical HIA & $578 \pm 126$ & $\mathbf{3.2} \pm 0.7$ & $68.798 \pm 20.136$ \\
 
\noalign{\smallskip}\hline\noalign{\smallskip}
Gaussian & $87.5$ 
 & Top-$k$ & $\mathbf{63} \pm 34$ & $17.1 \pm 9.2$ & $\mathbf{0.156} \pm 0.074$ \\
 & & Scaled HIA & $1126 \pm 1327$ & $\mathbf{6.2} \pm 7.3$ & $142.364 \pm 254.248$ \\
 & & Scaled TI & $1215 \pm 1407$ & $6.7 \pm 7.8$ & $14.256 \pm 35.297$ \\
\end{tabular}
\end{table}

\begin{table}[ht!]
\centering
\caption{Comparison of explanation size and computing time on CUB200.}
\label{tab:comparison_cub200}
\begin{tabular}{l c l ccc}
\hline\noalign{\smallskip}
Architecture & Acc. (\%) & ALE & Abs. Size & Rel. Size (\%) & Time (s) \\
\noalign{\smallskip}\hline\noalign{\smallskip}
ProtoPNet & $83.9$ 
 & Top-$k$ & $\mathbf{314} \pm 218$ & $15.7 \pm 10.9$ & $\mathbf{0.997} \pm 0.732$ \\
 & & TI & $1950 \pm 5860$ & $2.0 \pm 6.0$ & $213.616 \pm 666.449$ \\
 & & HIA & $1349 \pm 3769$ & $\mathbf{1.4} \pm 3.8$ & $2446.927 \pm 11363.445$ \\

\noalign{\smallskip}\hline\noalign{\smallskip}
ProtoPool & $82.5$ 
 & Top-$k$ & $\mathbf{67} \pm 44$ & $30.5 \pm 20.3$ & $\underline{\mathbf{0.195}} \pm 0.111$ \\
 & & TI & $893 \pm 522$ & $45.5 \pm 26.6$ & $7.502 \pm 3.505$ \\
 & & HIA & $272 \pm 315$ & $\mathbf{13.8} \pm 16.1$ & $7.666 \pm 12.973$ \\

\noalign{\smallskip}\hline\noalign{\smallskip}
PIP-Net & $82.7$
 &   Top-$k$ & $\underline{\mathbf{18}} \pm 38$ & $2.3 \pm 4.9$ & $\mathbf{1.031} \pm 0.910$ \\
 & & Sparse & $216 \pm 111$ & $28.1 \pm 14.4$ & $1.064 \pm 1.025$  \\
 & & Simplex & $955 \pm 339$ & $\underline{\mathbf{0.2}} \pm 0.1$ & $50.427 \pm 28.036$ \\

\noalign{\smallskip}\hline\noalign{\smallskip}
TesNet & $86.8$ 
 &   Top-$k$ & $58 \pm 156$ & $2.9 \pm 7.8$ & $2.078 \pm 1.947$ \\
 & & Cosine TI & & & \texttt{Timeout} \\
 & & Spherical HIA & & & \texttt{Timeout} \\

\noalign{\smallskip}\hline\noalign{\smallskip}
Gaussian & $81.4$ 
 &   Top-$k$ &  $95 \pm 211$ & $4.7 \pm 10.5$ & $0.340 \pm 0.510$ \\
 & & Scaled TI & & & \texttt{Timeout} \\
 & & Scaled HIA & & & \texttt{Timeout} \\

\end{tabular}
\end{table}

\section{Discussion}\label{sec:discussion}

By generalizing the ALE framework, we provide the first quantitative comparison of formal interpretability across diverse PBNs. 

PIP-Net yields the smallest absolute explanation sizes across all datasets (averaging just 18 to 103 prototypes) and boasts the fastest computation times ($\sim$0.07s). On CUB-200 in particular, its Top-$k$ ALE achieves the smallest absolute size across all architectures and datasets (18 $\pm$ 38). This establishes a clear architectural directive: models designed with non-negative, sparse linear heads are inherently more compatible with formal verification. Furthermore, applying the \textit{Simplex} paradigm yields the smallest \emph{relative} explanation size overall (0.1\%--0.2\%). This extreme compression hints that Simplex explanations scale well with the latent space dimensions.

In Gaussian models, the \textit{Scaled HIA} consistently provides the smallest relative explanation sizes (4.6\%--6.2\% on Oxford datasets), demonstrating that mapping probabilistic similarities into an intersection of Euclidean hyperspheres creates tightly bounded, subset-minimal explanations. However, this geometric precision comes at the cost of the longest computation times (up to 142s). On CUB-200, both Scaled TI and Scaled HIA exceed the time budget entirely, underscoring a fundamental scalability limitation of these paradigms as the number of classes grows.
Conversely, models operating on spherical manifolds via cosine similarity (TesNet) exhibit high variance. While the \textit{Spherical HIA} paradigm collapses the bounds -- achieving the smallest absolute size on Oxford Flowers (187) and highly competitive relative sizes on Oxford Pets (3.2\%) -- both Cosine TI and Spherical HIA time out on CUB-200, mirroring the Gaussian scalability failure and suggesting that the underlying cosine similarity inherently complicates spatial reasoning, requiring careful optimization to balance bound tightness and runtime.

While our extension of the ALE framework standardizes the mathematical guarantees, raw explanation sizes must be compared across architectures with careful consideration of structural confounders like the latent space dimensions.
Furthermore, we evaluate these architectures in their raw, fully parameterized states. Prototype pruning (or merging)---a common post-training step to reduce $|\prototypeindices|$---is orthogonal to our bounding algorithms but would mechanically reduce the absolute sizes of all ALEs. Therefore, when comparing different models, the \emph{Relative Size} metric provides a robust indicator of inherent architectural interpretability than absolute counts alone.

\section{Conclusion}
\label{sec:conclusion}
In this work, we have extended the ALE definitions to cover the rich landscape of modern prototype-based networks. Our experiments have allowed us to perform a quantitative comparison of models and highlighted design principles leading to a better formal interpretability.

\paragraph{Acknowledgment.}This publication was made possible by the use of the FactoryIA supercomputer, financially supported by the Ile-De-France Regional Council. This work was supported by the SAIF project, funded by the “France 2030” government investment plan managed by the French National Research Agency, under the reference ANR-23-PEIA-0006.

\bibliographystyle{splncs04}
\bibliography{mybibliography}
\end{document}